\documentclass[11pt]{article}
\usepackage{mathtools,tikz,parskip,algorithm,subcaption,amssymb,hyperref}
\usepackage[noend]{algpseudocode}
\usepackage{geometry}
\usepackage{times}
\newcommand{\alg}[2]{\begin{algorithm}[ht] \caption{#1} \label{alg:#2}
  \begin{algorithmic}[1]}
\newcommand{\ealg}{\end{algorithmic} \end{algorithm}}

\renewcommand{\b}[1]{\left[#1\right]}

\newcommand{\omt}[1]{}
\newcommand{\xhdr}[1]{\vspace{-0.2in} \paragraph*{\bf {#1}.}}
\def\eps{\varepsilon}

\def\R{\mathbb{R}}
\def\tran{^{\top}}
\def\s{|\sigma|}
\def\ones{{\bf e}}

\def\squarebox#1{\hbox to #1{\hfill\vbox to #1{\vfill}}}
\newcommand{\qed}{\hspace*{\fill}\vbox{\hrule\hbox{\vrule\squarebox{.667em}\vrule}\hrule}\smallskip}
\newenvironment{proof}{\noindent{\bf Proof:~~}}{\(\qed\)}

\newtheorem{theorem}{Theorem}[section]
\newtheorem{corollary}[theorem]{Corollary}
\newtheorem{lemma}[theorem]{Lemma}

\usetikzlibrary{positioning,decorations.pathmorphing,shapes,decorations.pathreplacing}
\newcommand{\p}[1]{\left(#1\right)}

\begin{document}
\title{Inherent Trade-Offs in the Fair Determination of Risk Scores}

\date{}
\author{Jon Kleinberg
\thanks{Cornell University}
\and
Sendhil Mullainathan
\thanks{Harvard University}
\and
Manish Raghavan
\thanks{Cornell University}
}

\maketitle

\begin{abstract}
Recent discussion in the public sphere about algorithmic classification
has involved tension between competing notions of what it means
for a probabilistic classification to be fair to different groups.  
We formalize three fairness conditions that lie at the heart of these
debates, and we prove that except in highly constrained special cases,
there is no method that can satisfy these three conditions simultaneously.
Moreover, even satisfying all three conditions approximately requires 
that the data lie in an approximate version of one of the constrained 
special cases identified by our theorem.
These results suggest some of the ways in which key notions of
fairness are incompatible with each other, and hence provide a
framework for thinking about the trade-offs between them.
\end{abstract}


\section{Introduction}

There are many settings in which a sequence of people comes before
a decision-maker, who must make a judgment about each based on 
some observable set of features.
Across a range of applications, these judgments are being carried out
by an increasingly wide spectrum of approaches ranging from human
expertise to algorithmic and statistical frameworks, as well as various
combinations of these approaches.

Along with these developments, 
a growing line of work has asked 
how we should reason about issues of bias and discrimination 
in settings where these algorithmic and statistical techniques, trained
on large datasets of past instances, play a significant role in the outcome.
Let us consider three examples where such issues arise, both to illustrate
the range of relevant contexts, and to surface
some of the challenges.

\xhdr{A set of example domains}
First, at various points in the criminal justice system, including
decisions about bail, sentencing, or parole, an officer of the court
may use quantitative {\em risk tools} to assess a defendant's probability
of recidivism --- future arrest --- based on their past history 
and other attributes.
Several recent analyses have asked whether such tools are mitigating or
exacerbating the sources of bias in the criminal justice system;
in one widely-publicized report, Angwin et al.
analyzed a commonly used statistical method for assigning risk
scores in the criminal justice system --- the COMPAS risk tool ---
and argued that it was biased against African-American defendants 
\cite{angwin-propublica-risk-scores,larson-angwin-propublica-analysis}.
One of their main contentions was that the tool's errors were asymmetric:
African-American defendants were more likely to be incorrectly labeled as
higher-risk than they actually were, while white defendants were more 
likely to be incorrectly labeled as lower-risk than they actually were.
Subsequent analyses raised methodological objections to this report,
and also observed that despite the COMPAS risk 
tool's errors, its estimates of the probability of recidivism are equally
well calibrated to the true outcomes for both African-American
and white defendants
\cite{propublica-analysis-doc,dieterich-northpointe-fairness,flores-re-propublica-fair,gong-algorithm-ethics-pt1}.

Second, in a very different domain, researchers have begun to analyze
the ways in which different genders and racial groups experience 
advertising and commercial content on the Internet differently
\cite{datta-google-ad-targeting,sweeney-google-ad-discrimination}.
We could ask, for example: if a male user and female user are 
equally interested in a particular product, does it follow that they're
equally likely to be shown an ad for it?
Sometimes this concern may have broader implications, for example if
women in aggregate are shown ads for lower-paying jobs.
Other times, it may represent a clash with a user's leisure interests:
if a female user interacting with an advertising platform
is interested in an activity that tends to have a male-dominated viewership,
like professional football,
is the platform as likely to show her an ad
for football as it is to show such an ad to an interested male user?

A third domain, again quite different from the previous two, is 
medical testing and diagnosis.
Doctors making decisions about a patient's treatment may rely on
tests providing probability estimates for different diseases and conditions.
Here too we can ask whether such decision-making is being applied
uniformly across different groups of patients
\cite{garb-medical-bias,williams-medical-discrimination}, and
in particular how medical tests may play a differential role for
conditions that vary widely in frequency between these groups.

\xhdr{Providing guarantees for decision procedures}
One can raise analogous questions in many other domains of
fundamental importance, including 
decisions about hiring, lending, or school admissions
\cite{whitehouse-big-data}, but we will focus on the three
examples above for the purposes of this discussion.
In these three example domains, a few structural
commonalities stand out. 
First, the algorithmic estimates
are often being used as ``input'' to a larger framework that makes
the overall decision --- a risk score provided to a
human expert in the legal and medical instances, and the output
of a machine-learning algorithm provided to a larger advertising platform
in the case of Internet ads.
Second, the underlying task is generally about classifying whether
people possess some relevant property:
recidivism, a medical condition, or interest in a product.
We will refer to people as being {\em positive instances} if
they truly possess the property, and {\em negative instances} if they do not.
Finally, the algorithmic estimates being provided for these questions
are generally not pure yes-no decisions,
but instead probability estimates about whether people constitute
positive or negative instances.

Let us suppose that we are concerned about how our decision procedure
might operate differentially between two groups of interest (such as
African-American and white defendants, or male and female users of
an advertising system).
What sorts of guarantees should we ask for 
as protection against potential bias?

A first basic goal in this literature is that the probability estimates 
provided by the algorithm should be {\em well-calibrated}:
if the algorithm identifies a set of people as having a probability $z$ of
constituting positive instances,
then approximately a $z$ fraction of this set
should indeed be positive instances
\cite{crowson-calibration-risk-scores,foster-asymptotic-calibration}.
Moreover, this condition should hold when applied separately in each
group as well \cite{flores-re-propublica-fair}. 
For example, if we are thinking in terms of potential differences
between outcomes for men and women, this means requiring that a $z$ fraction
of men and a $z$ fraction of women assigned a probability $z$ should possess
the property in question.

A second goal focuses on the people who constitute positive instances
(even if the algorithm can only imperfectly recognize them):
the average score received by people constituting positive instances
should be the same in each group.
We could think of this as {\em balance for the positive class},
since a violation of it would mean that people constituting positive
instances in one group receive consistently lower probability estimates
than people constituting positive instances in another group.
In our initial criminal justice example, for instance, one of the concerns
raised was that white defendants who went on to commit future crimes were
assigned risk scores corresponding to lower probability estimates
in aggregate; this is a violation of the condition here.
There is a completely analogous property with respect to negative instances,
which we could call {\em balance for the negative class}. These balance
conditions can be viewed as generalizations of the notions that both groups
should have equal false negative and false positive rates.

It is important to note that balance for the positive and negative classes,
as defined here, is distinct in crucial ways from the requirement 
that the average probability estimate globally 
over {\em all} members of the two groups be equal.  
This latter global requirement is a version of {\em statistical parity}
\cite{certifying-disparate-impact,three-naive-bayes-approaches,
classifying-without-discrimination,fairness-aware-regularization}.
In some cases statistical parity is a central goal (and in some it
is legally mandated), but the examples considered so far suggest that
classification and risk assessment
are much broader activities where statistical parity is often
neither feasible nor desirable.
Balance for the positive and negative classes, however, is a goal
that can be discussed independently of statistical parity,
since these two balance conditions
simply ask that once we condition on the ``correct'' answer for a person,
the chance of making a mistake on them should not depend on which group
they belong to.

\xhdr{The present work: Trade-offs among the guarantees}
Despite their different formulations, the calibration condition and 
the balance conditions for the positive and negative classes intuitively
all seem to be asking for variants of the same general goal --- that our
probability estimates should have the same effectiveness regardless of
group membership.  
One might therefore
hope that it would be feasible to achieve all of them
simultaneously.  

Our main result, however, is that these conditions are in general
incompatible with each other; they can only be simultaneously satisfied
in certain highly constrained cases.
Moreover, this incompatibility applies to {\em approximate} versions
of the conditions as well.

In the remainder of this section
we formulate this main result precisely, as a theorem 
building on a model that makes the discussion thus far more concrete. 

\subsection{Formulating the Goal}

Let's start with some basic definitions.
As above, we have a collection of people
each of whom constitutes either a
positive instance or a negative instance of the classification problem.
We'll say that the {\em positive class} consists of the people
who constitute positive instances, and the negative class consists of
the people who constitute negative instances.
For example, for criminal defendants, the positive class could consist
of those defendants who will be arrested again
within some fixed time window,
and the negative class could consist of those who will not.
The positive and negative classes thus represent the ``correct''
answer to the classification problem; our decision procedure does not
know them, but is trying to estimate them.

\xhdr{Feature vectors}
Each person has an associated {\em feature vector} $\sigma$, representing
the data that we know about them.  Let $p_\sigma$ denote the fraction
of people with feature vector $\sigma$ who belong to the positive class.
Conceptually, we will picture that while there is variation within
the set of people who have feature vector $\sigma$, this variation
is invisible to whatever decision procedure we apply;
all people with feature vector $\sigma$ are indistinguishable to the procedure.
Our model will assume that
the value $p_\sigma$ for each $\sigma$ is known to the 
procedure.\footnote{Clearly the case in which the value of $p_\sigma$ is unknown
is an important version of the problem as well; however, since our main results
establish strong limitations on what is achievable, these limitations
are only stronger because they apply even to the case of known $p_\sigma$.}

\xhdr{Groups}
Each person also belongs to one of two {\em groups}, labeled $1$ or $2$,
and we would like our decisions
to be unbiased with respect to the members of these two 
groups.\footnote{We focus on the case of two groups for simplicity
of exposition, but it is straightforward to extend all of our definitions
to the case of more than two groups.}
In our examples, the two groups could correspond to different races
or genders, or other cases where we want to look for the possibility of bias
between them.
The two groups have different distributions over feature vectors:
a person of group $t$ has a probability $a_{t\sigma}$ of exhibiting
the feature vector $\sigma$.
However, people of each group have the same probability $p_\sigma$
of belonging to the positive class provided their feature vector is $\sigma$.
In this respect, $\sigma$ contains all the relevant information available
to us about the person's future behavior; once we know $\sigma$, we do not
get any additional information from knowing their group as 
well.\footnote{As we will discuss in more detail below, 
the assumption that the group provides no additional information beyond
$\sigma$ does not restrict the generality of the model, since 
we can always consider instances in which people of different groups
never have the same feature vector $\sigma$, and hence $\sigma$
implicitly conveys perfect information about a person's group.}

\xhdr{Risk Assignments}
We say that an {\em instance} of our problem is specified by the
parameters above: a feature vector and a group for each person, with
a value $p_\sigma$ for each feature vector, and distributions
$\{a_{t\sigma}\}$ giving the frequency of the feature vectors in each group.

Informally, risk assessments are ways of dividing people up into sets
based on their
feature vectors $\sigma$ (potentially using randomization),
and then assigning each set a probability
estimate that the people in this set belong to the positive class.
Thus, we define a {\em risk assignment} to consist of a set of ``bins''
(the sets), where each bin is labeled with a {\em score} $v_b$ that
we intend to use as the probability for everyone assigned to bin $b$.
We then create a rule for assigning
people to bins based on their feature vector $\sigma$;
we allow the rule to divide people with a fixed feature vector $\sigma$
across multiple bins (reflecting the possible use of randomization).
Thus, the rule is specified by values $X_{\sigma b}$:
a fraction $X_{\sigma b}$ of all people with 
feature vector $\sigma$ are assigned to bin $b$.
Note that the rule does not have access to the group $t$ of the person
being considered, only their feature vector $\sigma$.  
(As we will see, this does not mean that the rule
is incapable of exhibiting bias between the two groups.)
In summary, a risk assignment is specified by a set of bins,
a score for each bin, and values $X_{\sigma b}$ that define a mapping from
people with feature vectors to bins.

\xhdr{Fairness Properties for Risk Assignments}
Within the model, we now express the three conditions
discussed at the outset, each reflecting
a potentially different notion of what it means for the
risk assignment to be ``fair.''

\begin{itemize}
\item[(A)] {\em Calibration within groups} requires that for
each group $t$, and each bin $b$ with associated score $v_b$,
the expected number of people from group $t$ in $b$ who belong
to the positive class should be a $v_b$ fraction of the expected number
of people from group $t$ assigned to $b$.
\item[(B)] {\em Balance for the negative class} requires that
the average score assigned to people of group 1 who belong to the
negative class should be the same as
the average score assigned to people of group 2 who belong to the
negative class.
In other words, the assignment of scores shouldn't be systematically
more inaccurate for
negative instances in one group than the other.
\item[(C)] {\em Balance for the positive class} symmetrically requires that
the average score assigned to people of group 1 who belong to the
positive class should be the same as
the average score assigned to people of group 2 who belong to the
positive class.
\end{itemize}

\xhdr{Why Do These Conditions Correspond to Notions of Fairness?}
All of these are natural conditions to impose on a risk assignment;
and as indicated by the discussion above, all of them have been proposed
as versions of fairness.
The first one essentially
asks that the scores mean what they claim to mean, 
even when considered separately in each group.
In particular, suppose a set of scores lack the first property 
for some bin $b$, and these scores
are given to a decision-maker; then if people of two different
groups both belong to bin $b$, 
the decision-maker has a clear incentive
to treat them differently, since the lack of calibration within groups
on bin $b$ means that these people have different 
aggregate probabilities of belonging to the positive class. 
Another way of stating the property of 
calibration within groups is to say that, 
conditioned on the bin to which an
individual is assigned, the likelihood that the individual is a member of the
positive class is independent of the group to which the individual belongs.
This means we are justified in treating people with the same score
comparably with respect to the outcome, rather than treating people
with the same score differently based on the group they belong to.

The second and third ask that if two individuals in different groups
exhibit comparable future behavior (negative or positive),
they should be treated comparably by the procedure.
In other words, a violation of, say, the second condition would
correspond to the members of the negative class in one group receiving
consistently higher scores than the members of the negative class
in the other group, despite the fact that the members of the negative
class in the higher-scoring group have done nothing to warrant
these higher scores.

We can also interpret some of the prior work around our earlier examples
through the lens of these conditions.
For example, in the analysis of the COMPAS risk tool for criminal
defendants, the critique by Angwin et al. focused on the risk tool's
violation of conditions (B) and (C); the counter-arguments 
established that it satisfies condition (A).
While it is clearly crucial for a risk tool to satisfy (A), 
it may still be important to know that it violates (B) and (C).
Similarly, to think in terms of the example of Internet advertising,
with male and female users as the two groups, 
condition (A) as before requires that our estimates of ad-click probability
mean the same thing in aggregate for men and women.
Conditions (B) and (C) are distinct; condition (C), for example, says
that a female user who genuinely wants to see a given ad should be
assigned the same probability as a male user who wants to see the ad.

\subsection{Determining What is Achievable: A Characterization Theorem}

When can conditions (A), (B), and (C) be simultaneously achieved?
We begin with two simple cases where it's possible.

\begin{itemize}
\item {\em Perfect prediction.} Suppose that for each feature vector
$\sigma$, we have either $p_\sigma = 0$ or $p_\sigma = 1$.
This means that we can achieve perfect prediction, since we know
each person's class label (positive or negative) for certain.
In this case, we can assign all feature vectors $\sigma$ with $p_\sigma = 0$
to a bin $b$ with score $v_b = 0$, and all $\sigma$ with $p_\sigma = 1$
to a bin $b'$ with score $v_{b'} = 1$.
It is easy to check that all three of the conditions (A), (B), and (C)
are satisfied by this risk assignment.
\item {\em Equal base rates.} Suppose, alternately, that
the two groups have the same fraction of members in the positive class;
that is, the average value of $p_\sigma$
is the same for the members of group 1 and group 2.
(We can refer to this as the {\em base rate} of the group
with respect to the classification problem.)
In this case, we can create a single bin $b$ with score equal
to this average value of $p_\sigma$, and we can assign everyone to bin $b$.
While this is not a particularly informative risk assignment,
it is again
easy to check that it satisfies fairness conditions (A), (B), and (C).
\end{itemize}

Our first main result establishes that these are in fact the only two cases
in which a risk assignment can achieve all three fairness 
guarantees simultaneously.  

\begin{theorem}
Consider an instance of the problem in which there is a 
risk assignment 
satisfying fairness conditions (A), (B), and (C).
Then the instance must either allow for perfect prediction
(with $p_\sigma$ equal to $0$ or $1$ for all $\sigma$) or 
have equal base rates.
\label{thm:exact}
\end{theorem}

Thus, in every instance that is more complex than the two cases
noted above, there will be some natural fairness condition that is
violated by any risk assignment.
Moreover, note that this result applies regardless of how the
risk assignment is computed; since our framework considers risk assignments
to be arbitrary functions from feature vectors to bins
labeled with probability estimates, 
it applies independently of the method --- algorithmic or otherwise ---
that is used to construct the risk assignment.

The conclusions of the first theorem 
can be relaxed in a continuous fashion
when the fairness conditions are only approximate.
In particular, for any $\eps > 0$ 
we can define $\eps$-approximate versions of
each of conditions (A), (B), and (C) (specified precisely in the next
section), each of which requires that the corresponding equalities
between groups hold only to within an error of $\eps$.
For any $\delta > 0$, we can also define a 
$\delta$-approximate version of the equal base rates condition
(requiring that the base rates of the two groups 
be within an additive $\delta$ of each other)
and a $\delta$-approximate 
version of the perfect prediction condition (requiring that 
in each group, the average of the expected scores assigned
to members of the positive class is at least $1 - \delta$;
by the calibration condition, this can be shown to imply a 
complementary bound on the average of the expected scores
assigned to members of the negative class).

In these terms, our approximate version of Theorem~\ref{thm:exact}
is the following.

\begin{theorem}
There is a continuous function $f$, with $f(x)$ going to $0$ as 
$x$ goes to $0$, so that the following holds.
For all $\eps > 0$, and 
any instance of the problem with a risk assignment
satisfying the $\eps$-approximate versions
of fairness conditions (A), (B), and (C),
the instance must satisfy either the $f(\eps)$-approximate version of
perfect prediction or the $f(\eps)$-approximate version of
equal base rates.
\label{thm:approx}
\end{theorem}

Thus, anything that approximately satisfies the fairness constraints
must approximately look like one of the two simple cases identified above.

Finally, in connection to Theorem \ref{thm:exact}, we note that
when the two groups have equal base rates, then one can ask for
the most accurate risk assignment that satisfies all three fairness
conditions (A), (B), and (C) simultaneously.
Since the risk assignment that gives the same score to everyone satisfies
the three conditions, we know that at least one such risk assignment exists;
hence, it is natural to seek to optimize over the set of all such assignments.
We consider this algorithmic question in the final technical section of
the paper.

To reflect a bit further on our main theorems and what they suggest,
we note that our intention in the present work
isn't to make a recommendation on how conflicts
between different definitions of fairness should be handled.
Nor is our intention to analyze which definitions of
fairness are violated in particular applications or datasets.
Rather, our point is to establish certain unavoidable trade-offs
between the definitions, regardless of the specific context
and regardless of the method used to compute risk scores.
Since each of the definitions reflect (and have been proposed as) natural
notions of what it should mean for a risk score to be fair,
these trade-offs suggest a striking implication:
that outside of narrowly delineated cases,
any assignment of risk scores can in principle be subject to
natural criticisms on the grounds of bias.
This is equally true whether the risk score is determined by an algorithm
or by a system of human decision-makers.

\xhdr{Special Cases of the Model}
Our main results, which place strong restrictions on
when the three fairness conditions can be simultaneously satisfied,
have more power when the underlying model of the input
is more general, since it means that the restrictions implied by the
theorems apply in greater generality.
However, it is also useful to note certain special cases of our model,
obtained by limiting the flexibility of certain parameters in 
intuitive ways.
The point is that our results apply {\em a fortiori} 
to these more limited special cases.

First, we have already observed one natural special case of our model:
cases in which, for each feature vector $\sigma$, only members
of one group (but not the other) can exhibit $\sigma$.
This means that $\sigma$ contains perfect information about group membership,
and so it corresponds to instances in which risk assignments would have
the potential to use knowledge of an individual's group membership.
Note that we can convert any instance of our problem into a new
instance that belongs to this special case as follows.
For each feature vector $\sigma$, we create two new feature
vectors $\sigma^{(1)}$ and $\sigma^{(2)}$; then, for each member of group $1$
who had feature vector $\sigma$, we assign them $\sigma^{(1)}$, and for
each member of group $2$
who had feature vector $\sigma$, we assign them $\sigma^{(2)}$.
The resulting instance has the property that each feature vector is
associated with members of only one group, but it preserves the
essential aspects of the original instance in other respects.

Second, 
we allow risk assignments in our model to split
people with a given feature vector $\sigma$ over several bins.
Our results also therefore apply to the natural special case of the model
with {\em integral} risk assignments,
in which all people with a given feature $\sigma$ must go to the same bin.

Third, our model is a generalization of binary classification, which only allows
for 2 bins. Note that although binary classification does not explicitly assign
scores, we can consider the probability that an individual belongs to the
positive class given that they were assigned to a specific bin to be the score
for that bin. Thus, our results hold in the traditional binary classification
setting as well.

\xhdr{Data-Generating Processes}
Finally, there is the question of where the data in an instance of
our problem comes from.  
Our results do not assume any particular process for generating
the positive/negative class labels, 
feature vectors, and group memberships; we simply assume that
we are given such a collection of values (regardless of where they
came from), and then our results address the existence or non-existence
of certain risk assignments for these values.

This increases the generality of our results, since it means that they 
apply to any process that produces data of the form described by our model.
To give an example of a natural generative model that would produce
instances with the structure that we need, 
one could assume
that each individual starts with a ``hidden'' class label 
(positive or negative), and a feature
vector $\sigma$ is then probabilistically generated for this individual
from a distribution that can depend on their class label 
and their group membership.
(If feature vectors produced for the two groups are disjoint from
one another, then the requirement that the value of $p_\sigma$
is independent of group membership given $\sigma$ necessarily holds.)
Since a process with this structure produces instances
from our model, 
our results apply to data that arises from such a generative process.

It is also interesting to note that the basic set-up of our model,
with the population divided across a set of feature vectors 
for which race provides no additional information, is in fact a very close match
to the information one gets from the output of a well-calibrated risk tool.
In this sense, one setting for our model would be the problem of
applying post-processing to the output of such a risk tool to 
ensure additional fairness guarantees.
Indeed, since much of the recent controversy about fair risk scores
has involved risk tools that are well-calibrated but lack the
other fairness conditions we consider, such an interpretation of
the model could be a useful way to think about how one might
work with these tools in the context of a broader system.

\subsection{Further Related Work}

Mounting concern over discrimination in machine learning has led to a
large body of new work seeking to better understand and prevent it.
Barocas and Selbst survey a range of ways in which data-analysis
algorithms can lead to discriminatory outcomes
\cite{big-data-disparate-impact}, and review articles by 
Romei and Ruggieri \cite{romei-discrimination-survey}
and Zliobaite \cite{zliobaite-survey} survey 
data-analytic and algorithmic methods for measuring discrimination.

Kamiran and Calders \cite{classifying-without-discrimination}
and Hajian and Domingo-Ferrer \cite{methodology-discrimination-prevention}
seek to modify datasets to remove any information that might permit
discrimination.
Similarly, Zemel et al. look to learn fair intermediate
representations of data while preserving information needed for
classification \cite{learning-fair-representations}.
Joseph et al. consider how fairness issues can arise during the process of
learning, modeling this using a multi-armed bandit framework
\cite{bandit-fairness}.

One common notion of fairness
is ``statistical parity'' -- equal fractions of each group should be
treated as belonging to the positive class
\cite{three-naive-bayes-approaches,
classifying-without-discrimination,fairness-aware-regularization}.
Recent papers have also considered approximate relaxations of 
statistical parity, motivated by the formulation of {\em disparate impact}
in the U.S. legal code 
\cite{certifying-disparate-impact,
bilal-zafar-learning-fair-classif}.
Work in these directions
has developed learning algorithms that penalize violations
of statistical parity
\cite{three-naive-bayes-approaches,fairness-aware-regularization}. 
As noted above, we consider definitions other than statistical parity
that take into account the class membership (positive or negative)
of the people being classified.

Dwork et al.~propose a framework based on a task-specific externally defined
similarity metric between individuals, seeking to achieve fairness through the
goal that ``similar people [be] treated similarly''
\cite{dwork-fairness-awareness}. 
They strive towards individual fairness,
which is a stronger notion of fairness than the 
definitions we use;
however, our approach shares some of the underlying motivation
(though not the specifics) in that our balance conditions for
the positive and negative classes also reflect the notion that similar
people should be treated similarly.

Much of the applied work on risk scores, as noted above, focuses
on calibration as a central goal 
\cite{crowson-calibration-risk-scores,dieterich-northpointe-fairness,flores-re-propublica-fair}.
In particular, responding to the criticism of their risk scores as
displaying asymmetric errors for different groups,
Dietrich et al.~note that empirically, both in their domain and in similar
settings, it is typically difficult to achieve symmetry in the error rates
across groups when base rates differ significantly.
Our formulation of the balance conditions for the positive and negative
classes, and our result showing the incompatibility of these conditions
with calibration, provides a theoretical basis for such observations.

In recent work concurrent with ours, Hardt et al.~consider the natural
analogues of our conditions (B) and (C), balance for the negative and
positive classes, in the case of classifiers that output binary ``yes/no''
predictions rather than real-valued scores as in our case
\cite{hardt-fairness}.
Since they do not require an analogue of calibration, it is possible
to satisfy the two balance constraints simultaneously, and they provide
methods for optimizing performance measures of the prediction rule
subject to satisfying these two constraints.
Also concurrent with our work and that of Hardt et al.,
Chouldechova \cite{chouldechova-fair-prediction} and 
Corbett-Davies et al. \cite{washington-post}
(and see also \cite{corbett-davies-working-paper})
consider binary
prediction subject to these same analogues of the balance conditions
for the negative and positive classes, together with a form of
calibration adapted to binary prediction (requiring that 
for all people given a positive label, the same fraction
of people in each group should truly belong to the positive class).
Among other results, they show that no classification rule satisfying
the required constraints is possible.
Finally, a recent paper of Friedler et al.~\cite{suresh-impossibility}
defines two axiomatic properties of feature generation and 
shows that no mechanism can be fair under these two properties.

\section{The Characterization Theorems}

Starting with the notation and definitions from the previous section,
we now give a proof of Theorem \ref{thm:exact}.

\xhdr{Informal overview}
Let us begin with a brief overview of the proof, 
before going into a more detailed version of it.  
For this discussion, 
let $N_t$ denote the number of people in group $t$, and 
$\mu_t$ be the number of people in group $t$ who belong to the positive class.

Roughly speaking, the proof proceeds in two steps.
First, consider a single bin $b$.
By the calibration condition, the expected total score given to 
the group-$t$ people in bin $b$ is equal to the expected number
of group-$t$ people in bin $b$ who belong to the positive class.
Summing over all bins, we find that the total score given to all
people in group $t$ (that is, the sum of the scores received by
everyone in group $t$) is equal to the total number of people
in the positive class in group $t$, which is $\mu_t$.

Now, let $x$ be the average score given to a member of the negative class, 
and let $y$ be the average score given to a member of the positive class.
By the balance conditions for the negative and positive classes,
these values of $x$ and $y$ are the same for both groups.

Given the values of $x$ and $y$, the total number of people in the
positive class $\mu_t$, and the total score given out to people
in group $t$ --- which, as argued above, is also $\mu_t$ ---
we can write the total score as 
$$(N - \mu_t) x + \mu_t y = \mu_t.$$
This defines a line for each group $t$ in the two variables $x$ and $y$, 
and hence we 
obtain a system of two linear equations (one for each group) in the 
unknowns $x$ and $y$.

If all three conditions --- calibration, and balance for the
two classes --- are to be satisfied, then we must be at a set
of parameters that represents a solution to the system of two equations.
If the base rates are equal, then $\mu_1 = \mu_2$ and hence the
two lines are the same; in this case, 
the system of equations is satisfied by any choice of $x$ and $y$.
If the base rates are not equal, then the two lines are distinct, and
they intersect only at the point $(x,y) = (0,1)$, which implies
perfect prediction --- an average score of $0$ for members of
the negative class and $1$ for members of the positive class.
Thus, the three conditions can be simultaneously satisfied if
and only if we have equal base rates or perfect prediction.

This concludes the overview of the proof; in the remainder of the
section we describe the argument at a more detailed level.

\xhdr{Definitions and notation}
Recall from our notation in the previous section that
an $a_{t \sigma}$ fraction of the people in group $t$ have
feature vector $\sigma$; we thus write $n_{t \sigma} = a_{t \sigma} N_t$
for the number of people in group $t$ with feature vector $\sigma$.
Many of the components of the risk assignment and its evaluation
can be written in terms of operations on 
a set of underlying matrices and vectors, which we begin by specifying.

\begin{itemize}
\item 
First, let $\s$ denote the number of feature vectors in the instance,
and let $p \in \R^{\s}$ be a vector indexed by the possible
feature vectors, with the coordinate in position $\sigma$ equal to $p_\sigma$.
For group $t$, 
let $n_t \in \R^{\s}$ also be a vector indexed by the possible
feature vectors, with the coordinate in position $\sigma$ equal 
to $n_{t \sigma}$.
Finally, it will be useful to have a representation of $p$ as a 
diagonal matrix; thus, let $P$ be a $\s \times \s$
diagonal matrix with $P_{\sigma \sigma} = p_{\sigma}$.
\item 
We now specify a risk assignment as follows.
The risk assignment involves a set of $B$ bins with associated scores;
let $v \in \R^B$ be a vector indexed by the bins,
with the coordinate in position $b$ equal to the score $v_b$ of bin $b$.
Let $V$ be a diagonal matrix version of $v$: it is a $B \times B$ matrix
with $V_{bb} = v_b$.
Finally, let $X$ be the $\s \times B$ matrix of $X_{\sigma b}$ values,
specifying the fraction of people with feature vector $\sigma$
who get mapped to bin $b$ under the assignment procedure.
\end{itemize}

There is an important point to note about the $X_{\sigma b}$ values.
If all of them are equal to $0$ or $1$, this corresponds to a procedure
in which all people with the same feature vector $\sigma$ get assigned
to the same bin.
When some of the $X_{\sigma b}$ values are not equal to $0$ or $1$, the
people with vector $\sigma$ are being divided among multiple bins.
In this case, there is an implicit randomization taking place
with respect to the positive and negative classes, and with respect
to the two groups, which we can think of as follows.
Since the procedure cannot distinguish among people with vector $\sigma$,
in the case that it distributes these people across multiple bins, the
subset of people with vector $\sigma$ who belong to the positive and negative
classes, and to the two groups,
are divided up randomly across these bins in proportions corresponding
to $X_{\sigma b}$.
In particular, if there are $n_{t \sigma}$ group-$t$
people with vector $\sigma$,
the expected number of these people who belong to the positive class
and are assigned to bin $b$ is $n_{t \sigma} p_\sigma X_{\sigma b}$.

Let us now proceed with the proof of Theorem \ref{thm:exact},
starting with the assumption that
our risk assignment satisfies conditions
(A), (B), and (C).

\xhdr{Calibration within groups}
We begin by working out some useful expressions in terms of the
matrices and vectors defined above.
We observe that $n_t\tran P$ is a vector in $\R^{\s}$ whose coordinate
corresponding to feature vector $\sigma$ 
equals the number of people 
in group $t$ who have feature vector $\sigma$ and belong to the positive class.
$n_t\tran X$ is a vector in $\R^B$ whose coordinate
corresponding to bin $b$ equals the expected number of people
in group $t$ assigned to bin $b$.

By further multiplying these vectors on the right, we get additional
useful quantities.  Here are two in particular:
\begin{itemize}
\item $n_t\tran X V$ is a vector in $\R^B$ whose coordinate corresponding
to bin $b$ equals the expected sum of the scores assigned to all group-$t$
people in bin $b$.  
That is, using the subscript $b$ to denote the coordinate corresponding
to bin $b$, we can write
$(n_t\tran X V)_b = v_b (n_t\tran X)_b$
by the definition of the diagonal matrix $V$.
\item $n_t\tran P X$ is a vector in $\R^B$ whose coordinate corresponding
to bin $b$ equals the expected number of group-$t$ people in the positive 
class who are placed in bin $b$.  
\end{itemize}

Now, condition (A), that the risk assignment is calibrated within groups,
implies that the two vectors above are equal coordinate-wise, and so we have
the following equation for all $t$:
\begin{equation}
  n_t\tran PX = n_t\tran XV
  \label{eq:unbiased}
\end{equation}

Calibration condition (A) also has an implication for the 
total score received by all people in group $t$.
Suppose we multiply the two sides of \eqref{eq:unbiased} on the right
by the vector $\ones \in \R^B$ whose coordinates are all $1$, obtaining

\begin{equation}
n_t\tran PX \ones = n_t\tran XV \ones.
\label{eq:unbiased-sum}
\end{equation}

The left-hand-side is the number of group-$t$ people in the positive
class.  The right-hand-side, which we can also write as
$n_t\tran X v$, is equal to the sum of the expected scores received by all 
group-$t$ people.
These two quantities are thus the same, and we write their common value
as $\mu_t$.

\xhdr{Fairness to the positive and negative classes}
We now want to write down vector equations corresponding to 
the fairness conditions (B) and (C) for the negative and positive classes.
First, recall that for the $B$-dimensional vector
$n_t\tran P X$, the coordinate corresponding
to bin $b$ equals the expected number of group-$t$ people in the positive 
class who are placed in bin $b$.  
Thus, to compute the sum of the expected scores 
received by all group-$t$ people
in the positive class, we simply need to take the inner product with
the vector $v$, yielding $n_t\tran P X v$.  Since $\mu_t$ is the total
number of group-$t$ people in the positive class, the average 
of the expected scores received 
by a group-$t$ person in the positive class is the ratio
$\displaystyle{\frac{1}{\mu_t} n_t\tran PXv}$.
Thus, condition (C), that members of the positive class should receive
the same average score in each group, can be written
\begin{equation}
  \frac{1}{\mu_1} n_1\tran PXv = \frac{1}{\mu_2} n_2\tran PXv
  \label{eq:guilty_fair}
\end{equation}
Applying strictly analogous reasoning but to the fractions 
$1 - p_\sigma$ of people in the negative class,
we can write condition (B),
that members of the negative class should receive
the same average score in each group, as
\begin{equation}
  \frac{1}{N_1-\mu_1} n_1\tran (I-P)Xv = \frac{1}{N_2-\mu_2} n_2\tran (I-P)Xv
  \label{eq:inn_fair}
\end{equation}

Using~\eqref{eq:unbiased}, we can rewrite~\eqref{eq:guilty_fair} to get
\begin{equation}
  \frac{1}{\mu_1} n_1\tran XVv = \frac{1}{\mu_2} n_2\tran XVv
  \label{eq:guilty_fair_simp}
\end{equation}
Similarly, we can rewrite~\eqref{eq:inn_fair} as
\begin{equation}
  \frac{1}{N_1-\mu_1} (\mu_1 - n_1\tran XVv) = \frac{1}{N_2-\mu_2} (\mu_2 - n_2\tran XVv)
  \label{eq:inn_fair_simp}
\end{equation}

\xhdr{The portion of the score received by the positive class}
We think of the ratios on the two sides of 
\eqref{eq:guilty_fair}, and equivalently \eqref{eq:guilty_fair_simp},
as the average of the expected
scores received by a member of the positive class in group $t$:
the numerator is the sum of the expected scores received by the members of the
positive class, and the denominator is the size of the positive class.
Let us denote this fraction by $\gamma_t$; we note that this is
the quantity $y$ used in the informal overview of the proof at the
start of the section.
By \eqref{eq:unbiased-sum}, we can alternately think of the denominator
as the sum of the expected scores received by all group-$t$ people.
Hence, the two sides of \eqref{eq:guilty_fair} and \eqref{eq:guilty_fair_simp}
can be viewed as representing the ratio of the sum of the expected
scores in the positive class of group $t$ to the sum of the expected scores
in group $t$ as a whole.
\eqref{eq:guilty_fair} requires that $\gamma_1 = \gamma_2$; let us
denote this common value by $\gamma$. 

Now, we observe that $\gamma = 1$ corresponds to a case in which 
the sum of the expected scores in just the positive class of group $t$
is equal to the sum of the expected scores in all of group $t$.
In this case, it must be that all members of the negative class are
assigned to bins of score $0$.
If any members of the positive class were assigned to a bin of score $0$,
this would violate the calibration condition (A); hence
all members of the positive class are assigned to bins of positive score.
Moreover, these bins of positive score contain no members of the negative
class (since they've all been assigned to bins of score $0$), and so
again by the calibration condition (A), the members of the positive class
are all assigned to bins of score $1$.
Finally, applying the calibration condition once more, it follows that
the members of the negative class all have feature vectors $\sigma$
with $p_\sigma = 0$ and 
the members of the positive class all have feature vectors $\sigma$
with $p_\sigma = 1$.
Hence, when $\gamma = 1$ we have perfect prediction.

Finally, we use our definition of $\gamma_t$ as 
$\displaystyle{\frac{1}{\mu_t} n_t\tran XVv}$, and the fact that
$\gamma_1 = \gamma_2 = \gamma$ to write \eqref{eq:inn_fair_simp} as 
\begin{align*}
  \frac{1}{N_1-\mu_1} (\mu_1 - \gamma \mu_1) &= \frac{1}{N_2-\mu_2} (\mu_2 - \gamma \mu_2) \\
  \frac{1}{N_1-\mu_1} \mu_1 (1 - \gamma) &= \frac{1}{N_2-\mu_2} \mu_2 (1 - \gamma) \\
  \frac{\mu_1/N_1}{1-\mu_1/N_1} (1 - \gamma) &= \frac{\mu_2/N_2}{1-\mu_2/N_2} (1 - \gamma)
\end{align*}
Now, this last equality implies that one of two things must be the case.
Either $1 - \gamma = 0$, in which case $\gamma = 1$ and we have 
perfect prediction; 
or 
$$\displaystyle{\frac{\mu_1/N_1}{1-\mu_1/N_1}}
= \displaystyle{\frac{\mu_2/N_2}{1-\mu_2/N_2}},$$ in which case 
$\mu_1/N_1 = \mu_2/N_2$
and we have equal base rates.
This completes the proof of Theorem \ref{thm:exact}.

\xhdr{Some Comments on the Connection to Statistical Parity}
Earlier we noted that conditions (B) and (C) --- the balance conditions
for the positive and negative classes --- are quite different from the
requirement of {\em statistical parity}, which asserts that the average
of the scores over {\em all} members of each group 
be the same.

When the two groups have equal base rates, then the risk assignment that
gives the same score to everyone in the population achieves statistical
parity along with conditions (A), (B), and (C).
But when the two groups do not have equal base rates, 
it is immediate to show that statistical parity is
inconsistent with both the calibration condition (A) and
with the conjunction of the two balance conditions (B) and (C).
To see the inconsistency of statistical parity with the
calibration condition, we take Equation 
\eqref{eq:unbiased} from the proof above,
sum the coordinates of the vectors on both sides, and
divide by $N_t$, the number of people in group $t$.
Statistical parity requires that the right-hand sides of the resulting
equation be the same for $t = 1, 2$, while the assumption that
the two groups have unequal base rates implies that the left-hand sides
of the equation must be different for $t = 1, 2$.
To see the inconsistency of statistical parity with the
two balance conditions (B) and (C), we simply observe that if
the average score assigned to the positive class and to the negative
class are the same in the two groups, then the average of the scores
over all members of the two groups cannot be the same provided they
do not contain the same proportion of positive-class and negative-class
members.

\section{The Approximate Theorem}


In this section we prove Theorem \ref{thm:approx}.
First, we must first give a precise specification
of the approximate fairness conditions:
\begin{align*}
  (1-\varepsilon) [n_t\tran XV]_b &\le [n_t\tran PX]_b \le (1-\varepsilon)
  [n_t\tran XV]_b \tag{A'} \label{eq:approx_su} \\
  (1-\varepsilon) \p{\frac{1}{N_2-\mu_2}} n_t\tran (I-P) Xv &\le
  \p{\frac{1}{N_1-\mu_1}} n_t\tran (I-P)Xv \le (1+\varepsilon)
  \p{\frac{1}{N_2-\mu_2}} n_t\tran (I-P) Xv \tag{B'} \label{eq:approx_in} \\
  (1-\varepsilon) \p{\frac{1}{\mu_2}} n_t\tran P Xv &\le
  \p{\frac{1}{\mu_1}} n_t\tran PXv \le (1+\varepsilon)
  \p{\frac{1}{\mu_2}} n_t\tran P Xv \tag{C'} \label{eq:approx_gu}
\end{align*}
For \eqref{eq:approx_in} and \eqref{eq:approx_gu}, we also require that these
hold when $\mu_1$ and $\mu_2$ are interchanged.

We also specify the approximate versions of perfect prediction and equal base
rates in terms of $f(\varepsilon)$, which is a function that goes to 0 as
$\varepsilon$ goes to 0.
\begin{itemize}
  \item {\em Approximate perfect prediction.} $\gamma_1 \ge 1 - f(\varepsilon)$
    and $\gamma_2 \ge 1 - f(\varepsilon)$
  \item {\em Approximately equal base rates.} $|\mu_1/N_1 - \mu_2/N_2| \le
    f(\varepsilon)$
\end{itemize}

A brief overview of the proof of Theorem \ref{thm:approx} is
as follows.  It proceeds 
by first establishing an approximate form of Equation 
\eqref{eq:unbiased} above, which implies that the total 
expected score assigned
in each group is approximately equal to the total size of the positive class.
This in turn makes it possible to formulate approximate forms of 
Equations \eqref{eq:guilty_fair} and \eqref{eq:inn_fair}.
When the base rates are close together, the approximation is too loose to derive
bounds on the predictive power; but this is okay since in this case we have
approximately equal base rates. Otherwise, when the base rates differ
significantly, we show that most of the expected score must be assigned to the
positive class, giving us approximately perfect prediction.

The remainder of this section provides the full details of the proof.

\xhdr{Total scores and the number of people in the positive class}
First, we will show that
the total score for each group is approximately $\mu_t$, the number of people in
the positive class. Define $\hat \mu_t =
n_t\tran Xv$. Using
\eqref{eq:approx_su}, we have
\begin{align*}
  \hat \mu_t &= n_t\tran Xv \\
  &= n_t\tran XVe \\
  &= \sum_{b=1}^B [n_t\tran PX]_b \\
  &\le (1+\varepsilon) \sum_{b=1}^B [n_t\tran PX]_b \\
  &= (1+\varepsilon) n_t\tran PXe \\
  &= (1+\varepsilon) \mu_t
\end{align*}
Similarly, we can lower bound $\hat \mu_t$ as
\begin{align*}
  \hat \mu_t &= \sum_{b=1}^B [n_t\tran PX]_b \\
  &\ge (1-\varepsilon) \sum_{b=1}^B [n_t\tran PX]_b \\
  &= (1-\varepsilon) \mu_t
\end{align*}
Combining these, we have
\begin{equation}
  (1-\varepsilon) \mu_t \le \hat \mu_t \le (1+\varepsilon) \mu_t.
  \label{eq:approx_mu}
\end{equation}

\xhdr{The portion of the score received by the positive class}
We can use \eqref{eq:approx_gu} to show that $\gamma_1 \approx \gamma_2$.
Recall that $\gamma_t$, the average of the expected scores assigned
to members of the positive class in group $t$,
is defined as $\gamma_t =
\frac{1}{\mu_t} n_tPXv$. Then, it follows trivially from \eqref{eq:approx_gu}
that
\begin{equation}
  (1-\varepsilon) \gamma_2 \le \gamma_1 \le (1+\varepsilon) \gamma_2.
  \label{eq:approx_gamma}
\end{equation}

\xhdr{The relationship between the base rates}
We can apply this to \eqref{eq:approx_in} to relate $\mu_1$ and $\mu_2$, using
the observation that the score not received by people of the positive class must fall
instead to people of the negative class. Examining the left inequality of
\eqref{eq:approx_in}, we have
\begin{align*}
  (1-\varepsilon) \p{\frac{1}{N_2-\mu_2}} n_t\tran (I-P) Xv &= 
  (1-\varepsilon) \p{\frac{1}{N_2-\mu_2}} (n_t\tran Xv - n_t\tran PXv) \\
  &= (1-\varepsilon) \p{\frac{1}{N_2-\mu_2}} (\hat \mu_2 - \gamma_2 \mu_2) \\
  &\ge (1-\varepsilon) \p{\frac{1}{N_2-\mu_2}} ((1-\varepsilon) \mu_2 - \gamma_2
  \mu_t) \\
  &= (1-\varepsilon) \p{\frac{\mu_2}{N_2-\mu_2}} (1-\varepsilon - \gamma_2) \\
  &\ge (1-\varepsilon) \p{\frac{\mu_2}{N_2-\mu_2}} \p{1-\varepsilon -
  \frac{\gamma_1}{1-\varepsilon}} \\
  &= (1-2\varepsilon + \varepsilon^2 - \gamma_1)\p{\frac{\mu_2}{N_2-\mu_2}}
\end{align*}
Thus, the left inequality of \eqref{eq:approx_in} becomes
\begin{equation}
  (1-2\varepsilon + \varepsilon^2 - \gamma_1)\p{\frac{\mu_2}{N_2-\mu_2}} \le
  \p{\frac{1}{N_1-\mu_1}} n_t\tran (I-P)Xv
\end{equation}
By definition, $\hat \mu_1 = n_t\tran Xv$ and $\gamma_t \mu_t = n_t\tran PXv$,
so this becomes
\begin{equation}
  (1-2\varepsilon + \varepsilon^2 - \gamma_1)\p{\frac{\mu_2}{N_2-\mu_2}} \le
  \p{\frac{1}{N_1-\mu_1}} (\hat \mu_1 - \gamma_1 \mu_1)
  \label{eq:aprrox_mu_ratio}
\end{equation}

\xhdr{If the base rates differ}
Let $\rho_1$ and $\rho_2$ be the respective base rates, i.e. $\rho_1 =
\mu_1/N_1$ and $\rho_2 = \mu_2/N_2$. Assume that $\rho_1 \le \rho_2$ (otherwise
we can switch $\mu_1$ and $\mu_2$ in the above analysis), and assume towards
contradiction that the base rates differ by at least $\sqrt{\varepsilon}$,
meaning $\rho_1 + \sqrt{\varepsilon} < \rho_2$. Using
\eqref{eq:aprrox_mu_ratio},
\begin{align*}
  \frac{\rho_1 + \sqrt{\varepsilon}}{1-\rho_1 = \sqrt{\varepsilon}} &\le
  \frac{\rho_2}{1-\rho_2} \\
  & \le \p{\frac{1+\varepsilon - \gamma_1}{1 - 2\varepsilon + \varepsilon^2 -
  \gamma_1}}\p{\frac{\rho_1}{1-\rho_1}} \\
  (\rho_1 + \sqrt{\varepsilon})(1-\rho_1)(1 - 2\varepsilon + \varepsilon^2 -
  \gamma_1) &\le \rho_1 (1-\rho_1 - \sqrt{\varepsilon})(1+\varepsilon -
  \gamma_1) \\
  (\rho_1 + \sqrt{\varepsilon})(1-\rho_1)(1-2\varepsilon) - \rho_1(1-\rho_1 -
  \sqrt{\varepsilon})(1+\varepsilon) &\le \gamma_1 \b{(\rho_1 +
  \sqrt{\varepsilon})(1-\rho_1) - \rho_1(1-\rho_1 - \sqrt{\varepsilon})} \\
  \rho_1[(1-\rho_1)(1-2\varepsilon) - (1-\rho_1-\sqrt{\varepsilon})(1+\varepsilon)]
  + \sqrt{\varepsilon}(1-\rho_1)(1-2\varepsilon) &\le \gamma_1
  [\sqrt{\varepsilon}(1-\rho_1) + \sqrt{\varepsilon} \rho_1] \\
  \rho_1(-2\varepsilon + 2\varepsilon \rho_1 - \varepsilon +
  \varepsilon \rho_1 + \sqrt{\varepsilon} + \varepsilon\sqrt{\varepsilon}) +
  \sqrt{\varepsilon} (1 - 2\varepsilon - \rho_1 + 2\varepsilon \rho_1) &\le
  \gamma_1 \sqrt{\varepsilon} \\
  \rho_1(-3\varepsilon + 3\varepsilon \rho_1 + \sqrt{\varepsilon} + \varepsilon
  \sqrt{\varepsilon} - \sqrt{\varepsilon} + 2\varepsilon \sqrt{\varepsilon}) +
  \sqrt{\varepsilon}(1-2\varepsilon)
  &\le \gamma_1 \sqrt{\varepsilon} \\
  \varepsilon \rho_1(-3 + 3 \rho_1 + 3 \sqrt{\varepsilon}) +
  \sqrt{\varepsilon}(1-2\varepsilon) &\le \gamma_1
  \sqrt{\varepsilon} \\
  3\varepsilon \rho_1(-1 + \rho_1) + \sqrt{\varepsilon}(1-2\varepsilon) &\le
  \gamma_1 \sqrt{\varepsilon} \\
  1 - 2\varepsilon - 3\sqrt{\varepsilon}\rho_1(1-\rho_1) &\le \gamma_1 \\
  1 - \sqrt{\varepsilon}\p{2\sqrt{\varepsilon} + \frac{3}{4}} &\le \gamma_1
\end{align*}
Recall that $\gamma_2 \ge \gamma_1 (1-\varepsilon)$, so
\begin{align*}
  \gamma_2 &\ge (1-\varepsilon) \gamma_1 \\
  &\ge (1-\varepsilon)\p{1-\sqrt{\varepsilon} \p{2 \sqrt{\varepsilon} +
  \frac{3}{4}}} \\
  &\ge 1 - \varepsilon - \sqrt{\varepsilon} \p{2 \sqrt{\varepsilon} +
  \frac{3}{4}} \\
  &= 1 - \sqrt{\varepsilon} \p{3\sqrt{\varepsilon} + \frac{3}{4}}
\end{align*}
Let $f(\varepsilon) = \sqrt{\varepsilon} \max(1,
3\sqrt{\varepsilon} + 3/4)$. Note that we assumed that $\rho_1$ and $\rho_2$
differ by an additive $\sqrt{\varepsilon} \le f(\varepsilon)$. Therefore if the
$\varepsilon$-fairness conditions are met and the base rates are not within an
additive $f(\varepsilon)$, then $\gamma_1 \ge 1-f(\varepsilon)$ and $\gamma_2
\ge 1-f(\varepsilon)$. This completes the proof of Theorem \ref{thm:approx}.

\section{Reducing Loss with Equal Base Rates}

In a risk assignment, we would like as much of the score as possible
to be assigned to members of the positive class.
With this in mind, if an individual receives a score of $v$, we
define their {\em individual loss} to be $v$ if they belong
to the negative class, and $1 - v$ if they belong to the positive class.
The loss of the risk assignment in group $t$ is then the sum of the 
expected individual losses to each member of group $t$.
In terms of the matrix-vector products used in the proof 
of Theorem \ref{thm:exact}, one can show that the loss for group $t$ 
may be written as
\begin{align*}
  \ell_t(X) &= n_t\tran (I-P)Xv + (\mu_t - n_t\tran PXv) \\
  &= 2(\mu_t - n_t\tran PXv),
\end{align*}
and the total loss is just the weighted sum of the losses for each group.

Now, let us say that a 
{\em fair assignment} is one that satisfies our three conditions
(A), (B), and (C).  
As noted above, 
when the base rates in the two groups are equal, 
the set of fair assignments is non-empty, 
since the calibrated risk assignment that places
everyone in a single bin is fair.
We can therefore ask, in the case of equal base rates, whether
there exists a fair assignment whose loss is strictly less than 
that of the trivial one-bin assignment.
It is not hard to show that this is possible if and only if there
is any assignment using more than one bin; we will call such an
assignment a {\em non-trivial assignment}.

Note that the assignment that minimizes loss is simply the one that
assigns each $\sigma$ to a separate bin with a score of $p_{\sigma}$,
meaning $X$ is the identity matrix. While this assignment, which we
refer to as the identity assignment $I$, is well-calibrated, it may
violate fairness conditions (B) and (C).  It is not hard to show that the
loss for any other assignment is strictly greater than the loss for
$I$. As a result, unless the identity assignment happens to be fair,
every fair assignment must have larger loss than that of $I$, forcing a
tradeoff between performance and fairness.

%

\subsection{Characterization of Well-Calibrated Solutions}

To better understand the space of feasible solutions, suppose we drop the
fairness conditions (B) and (C) for now and study risk assignments that are
simply well-calibrated, satisfying (A). As in the proof of Theorem
\ref{thm:exact}, we write $\gamma_t$ for the average of the expected
scores assigned 
to members of the positive class in group $t$, and we define the
{\em fairness difference} to be $\gamma_1 - \gamma_2$. If this is nonnegative,
we say the risk assignment {\em weakly favors} group 1; if it is nonpositive, it
weakly favors group 2. Since a risk assignment is fair if and only if $\gamma_1
= \gamma_2$, it is fair if and only if the fairness difference is 0.

We wish to characterize when non-trivial fair risk assignments are possible.
First, we observe that without the fairness requirements, the set of possible
fairness differences under well-calibrated assignments is an interval.

\begin{lemma}
  If group 1 and group 2 have equal base rates, then for any two non-trivial
  well-calibrated risk assignments with fairness differences $d_1$ and $d_2$ and
  for any $d_3 \in [d_1, d_2]$, there exists a non-trivial well-calibrated risk
  assignment with fairness difference $d_3$.
  \label{lem:diff_convex}
\end{lemma}
\begin{proof}
The basic idea is that we can effectively take convex combinations of
well-calibrated assignments to produce any well-calibrated assignment ``in
between'' them.  We carry this out as follows.

  Let $X^{(1)}$ and $X^{(2)}$ be the allocation matrices for assignments with
  fairness differences $d_1$ and $d_2$ respectively, where $d_1 < d_2$. Choose
  $\lambda$ such that $\lambda d_1 + (1-\lambda) d_2 = d_3$, meaning $\lambda =
  (d_2-d_3)/(d_2-d_1)$. Then, $X^{(3)} = [\lambda X^{(1)} ~~~ (1-\lambda)
  X^{(2)}]$ is a nontrivial well-calibrated assignment with fairness difference
  $d_3$.

  First, we observe that $X^{(3)}$ is a valid assignment because each row sums
  to 1 (meaning everyone from every $\sigma$ gets assigned to a bin), since each
  row of $\lambda X^{(1)}$ sums to $\lambda$ and each row of $(1-\lambda)
  X^{(2)}$ sums to $(1-\lambda)$. Moreover, it is nontrivial because every
  nonempty bin created by $X^{(1)}$ and $X^{(2)}$ is a nonempty bin under
  $X^{(3)}$.

  Let $v^{(1)}$ and $v^{(2)}$ be the respective bin labels for assignments
  $X^{(1)}$ and $X^{(2)}$. Define $\displaystyle{
    v^{(3)} = \begin{bmatrix}
      v^{(1)} \\
      v^{(2)}
    \end{bmatrix}
  }$.

  Finally, let $V^{(3)} = \text{diag}(v^{(3)})$. Define $V^{(1)}$ and
  $V^{(2)}$ analogously.
  Note that
  $\displaystyle{
    V^{(3)} = \begin{bmatrix}
      V^{(1)} & 0 \\
      0 & V^{(2)}
    \end{bmatrix}
  }$.

  We observe that $X^{(3)}$ is calibrated because
  \begin{align*}
    n_t\tran PX^{(3)} &= n_t\tran P [\lambda X^{(1)} ~~~ (1-\lambda)
    X^{(2)}] \\
    &= [\lambda n_t\tran PX^{(1)} ~~~ (1-\lambda) n_t\tran PX^{(2)}] \\
    &= [\lambda n_t\tran X^{(1)} V^{(1)} ~~~ (1-\lambda) n_t\tran
    X^{(2)}V^{(2)}] \\
    &= n_t\tran [\lambda X^{(1)} ~~~ (1-\lambda) X^{(2)}] V^{(3)} \\
    &= n_t\tran X^{(3)}V^{(3)}
  \end{align*}

  Finally, we show that the fairness difference is $d_3$. Let
  $\gamma_1^{(1)}$ and $\gamma_2^{(1)}$ be the portions of the total 
  expected score
  received by the positive class from each group respectively.
  Define
  $\gamma_1^{(2)}, \gamma_2^{(2)}, \gamma_1^{(3)}, \gamma_2^{(3)}$ similarly.
  \begin{align*}
    \gamma_1^{(3)} - \gamma_2^{(3)} &= \frac{1}{\mu} n_1\tran
    PX^{(3)}v^{(3)} -
    \frac{1}{\mu} n_2\tran PX^{(3)}v^{(3)} \\
    &= \frac{1}{\mu} (n_1\tran - n_2\tran) PX^{(3)}v^{(3)} \\
    &= \frac{1}{\mu} (n_1\tran - n_2\tran) P[\lambda X^{(1)}v^{(1)} ~~~
    (1-\lambda) X^{(2)} v^{(2)}] \\
    &= \frac{1}{\mu} (\lambda (n_1\tran - n_2\tran) PX^{(1)}v^{(1)} + 
    (1-\lambda) (n_1\tran - n_2\tran) X^{(2)} v^{(2)}]) \\
    &= \lambda(\gamma_1^{(1)} - \gamma_2^{(1)}) + (1-\lambda)
    (\gamma_1^{(2)} - \gamma_2^{(2)}) \\
    &= \lambda d_1 + (1-\lambda) d_2 \\
    &= d_3
  \end{align*}
\end{proof}

\begin{corollary}
  There exists a non-trivial fair assignment if and only if there exist
  non-trivial well-calibrated assignments $X^{(1)}$ and $X^{(2)}$ such that
  $X^{(1)}$ weakly favors group 1 and $X^{(2)}$ weakly favors group 2.
\end{corollary}
\begin{proof}
  If there is a non-trivial fair assignment, then it weakly favors both group 1
  and group 2, proving one direction.

  To prove the other direction, observe that the fairness differences $d_1$ and
  $d_2$ of $X^{(1)}$ and $X^{(2)}$ are nonnegative and nonpositive respectively.
  Since the set of fairness differences achievable by non-trivial
  well-calibrated assignments is an interval by Lemma \ref{lem:diff_convex}, 
  there
  exists a non-trivial well-calibrated assignment with fairness difference 0,
  meaning there exists a non-trivial fair assignment.
\end{proof}

It is an open question whether there is a polynomial-time algorithm to find
a fair assignment of minimum loss, or even to determine whether a
non-trivial fair solution exists.

\subsection{NP-Completeness of Non-Trivial Integral Fair Risk Assignments}

As discussed in the introduction, 
risk assignments in our model are allowed to split 
people with a given feature vector $\sigma$ over several bins;
however, it is also of interest to consider the special case
of {\em integral} risk assignments,
in which all people with a given feature $\sigma$ must go to the same bin.
For the case of equal base rates, we can show that determining
whether there is a non-trivial integral fair assignment is NP-complete.
The proof uses a reduction from the Subset Sum problem 
and is given in the Appendix.

The basic idea of the reduction is as follows. 
We have an instance of Subset Sum with numbers
$w_1, \ldots, w_m$ and a target number $T$; the question 
is whether there is a subset of the $w_i$'s that sums to $T$.
As before, $\gamma_t$ denotes the average of the expected scores received by
members of the positive class in group $t$.
We first ensure that there is
exactly one non-trivial way to allocate the people of group 1, allowing
us to control $\gamma_1$. The fairness conditions then require that
$\gamma_2 = \gamma_1$, which we can use to encode the target value in
the instance of Subset Sum. For every input number $w_i$
in the Subset Sum instance, we
create $p_{\sigma_{2i-1}}$ and $p_{\sigma_{2i}}$, close to each other
in value and far from all other $p_\sigma$ values, such that grouping
$\sigma_{2i-1}$ and $\sigma_{2i}$ together into a bin corresponds to
choosing $w_i$ for the subset, while not grouping them corresponds to not
taking $w_i$. This ensures that group 2 can be assigned
with the correct value of $\gamma_2$ if and only if there is a
solution to the Subset Sum instance.

\section{Conclusion}

In this work we have formalized three fundamental conditions for
risk assignments to individuals, each of which has been proposed as
a basic measure of what it means for the risk assignment to be fair.
Our main results show that except in highly constrained special cases,
it is not possible to satisfy these three constraints simultaneously;
and moreover, a version of this fact holds in an approximate sense as well.

Since these results hold regardless of the
method used to compute the risk assignment, it can be phrased in fairly
clean terms in a number of domains where the trade-offs among these
conditions do not appear to be well-understood.
To take one simple example, suppose we want to determine the risk that
a person is a carrier for a disease $X$, and suppose that a higher fraction
of women than men are carriers.
Then our results imply that
in any test designed to estimate the probability that someone is a 
carrier of $X$, at least one of the following undesirable properties must hold:
(a) the test's probability estimates are systematically skewed upward or
downward for at least one gender;
or (b) the test assigns a higher average risk estimate to healthy people 
(non-carriers) in one gender than the other; 
or (c) the test assigns a higher average risk estimate to carriers of
the disease in one gender than the other.
The point is that this trade-off among (a), (b), and (c) is not a fact
about medicine; it is simply a fact about risk estimates when the 
base rates differ between two groups.

Finally, we note that our results suggest a number of interesting directions
for further work.  First, when the base rates between the two underlying
groups are equal, our results do not resolve the computational tractability
of finding the most accurate risk assignment, subject to our three
fairness conditions, when the people with a given feature vector 
can be split across multiple bins.  (Our NP-completeness result applies
only to the case in which everyone with a given feature vector
must be assigned to the same bin.)  Second, there may be a number of
settings in which the cost (social or otherwise) of false positives
may differ greatly from the cost of false negatives.
In such cases, we could imagine searching for risk assignments 
that satisfy the calibration condition together with only one
of the two balance conditions, corresponding to the class for whom
errors are more costly.  Determining when two of our three 
conditions can be simultaneously satisfied in this way
is an interesting open question.
More broadly, determining how the trade-offs discussed here can be
incorporated into broader families of proposed fairness conditions
suggests interesting avenues for future research.

\bibliographystyle{plain}
\bibliography{refs}

\section*{Appendix: NP-Completeness of Non-Trivial Integral Fair Risk Assignments}

We can reduce to the integral assignment problem, parameterized by $a_{1\sigma},
a_{2\sigma}$, and $p_{\sigma}$, from subset sum as follows.

Suppose we have an instance of the subset sum problem specified by 
$m$ numbers $w_1, \dots, w_m$ and a target $T$; the goal is to
determine whether a subset of the $w_i$ add up to $T$.
We create an instance of
the integral assignment problem with $\sigma_1, \dots, \sigma_{2m+2}$.
$a_{1,\sigma_i} = 1/2$ if $i \in \{2m+1, 2m+2\}$ and 0 otherwise.
$a_{2,\sigma_i} = 1/(2m)$ if $i \le 2m$ and 0 otherwise. We make the following
definitions:
\begin{align*}
  \hat w_i &= w_i/(T m^4) \\
  \varepsilon_i &= \sqrt{\hat w_i / 2} \\
  p_{\sigma_{2i-1}} &= i/(m+1) - \varepsilon_i \tag{$1 \le i \le m$} \\
  p_{\sigma_{2i}} &= i/(m+1) + \varepsilon_i \tag{$1 \le i \le m$} \\
  \gamma &= 1/m \sum_{i=1}^{2m} p_{\sigma_i}^2 - 1/m^5 \\
  p_{\sigma_{2m+1}} &= (1 - \sqrt{2\gamma-1})/2 \\
  p_{\sigma_{2m+2}} &= (1 + \sqrt{2\gamma-1})/2
\end{align*}

With this definition, the subset sum instance has a solution if and only if the
integral assignment instance given by $a_{1,\sigma}, a_{2,\sigma}, p_{\sigma_1},
\dots, p_{\sigma_{2m+2}}$ has a solution.

Before we prove this, we need the following lemma.

\begin{lemma}
  For any $z_1, \dots, z_k \in \R$,
  \[
    \sum_{i=1}^k z_i^2 - \frac{1}{k} \p{\sum_{i=1}^m z_i}^2 =
    \frac{1}{k} \sum_{i < j}^k (z_i - z_j)^2
  \]
  \label{lem:square_diff}
\end{lemma}
\begin{proof}
  \begin{align*}
    \sum_{i=1}^k z_i^2 - \frac{1}{k} \p{\sum_{i=1}^m z_i}^2 &= \sum_{i=1}^k
    z_i^2 - \frac{1}{k} \p{\sum_{i=1}^k z_i^2 + 2 \sum_{i < j}^k z_i z_j} \\
    &= \frac{k-1}{k} \sum_{i=1}^k z_i^2 - \frac{2}{k} \sum_{i < j}^k z_i z_j \\
    &= \frac{1}{k} \sum_{i < j}^k (z_i^2 + z_j^2) - \frac{2}{k} \sum_{i < j}^k
    z_i z_j \\
    &= \frac{1}{k} \sum_{i < j}^k z_i^2 - 2 z_i z_j + z_j^2 \\
    &= \frac{1}{k} \sum_{i < j}^k (z_i - z_j)^2
  \end{align*}
\end{proof}

Now, we can prove that the integral assignment problem is NP-hard.

\begin{proof}
  First, we observe that for any nontrivial solution to the integral assignment
  instance, there must be two bins $b \ne b'$ such that $X_{\sigma_{2m+1},b} =
  1$ and $X_{\sigma_{2m+2},b'} = 1$. In other words, the people with
  $\sigma_{2m+1}$ and $\sigma_{2m+2}$ must be split up. If not, then all the
  people of group 1 would be in the same bin, meaning that bin must be labeled
  with the base rate $\rho_1 = 1/2$. In order to maintain fairness, the same
  would have to be done for all the people of group 2, resulting in the trivial
  solution. Moreover, $b$ and $b'$ must be labeled $(1 \pm \sqrt{2\gamma-1})/2$
  respectively because those are the fraction of people of group 1 in those bins
  who belong to the positive class.

  This means that $\gamma_1 = 1/\rho \cdot (a_{1,\sigma_{2m+1}}
  p_{\sigma_{2m+1}}^2 + a_{1,\sigma_{2m+2}} p_{\sigma_{2m+2}}^2) =
  p_{\sigma_{2m+1}}^2 + p_{\sigma_{2m+2}}^2 = \gamma$ as defined above. We know
  that a well-calibrated assignment is fair only if $\gamma_1 =
  \gamma_2$, so we know $\gamma_2 = \gamma$.

  Next, we observe that $\rho_2 = \rho_1 = 1/2$ because all of the positive
  $a_{2,\sigma}$'s are $1/(2m)$, so $\rho_2$ is just the average of
  $\{p_{\sigma_1}, \dots, p_{\sigma_{2m}}\}$, which is $1/2$ by symmetry.

  Let $Q$ be the partition of $[2m]$ corresponding to the assignment, meaning
  that for a given $q \in Q$, there is a bin $b_q$ containing all people with
  $\sigma_i$ such that $i \in q$. The label on that bin is
  \begin{align*}
    v_q &= \frac{\sum_{i \in q} a_{2,\sigma_i} p_{\sigma_i}}{\sum_{i \in q}
    a_{2,\sigma_i}} \\
    &= \frac{1/(2m) \sum_{i \in q} p_{\sigma_i}}{|q|/(2m)} \\
    &= \frac{1}{|q|} \sum_{i \in q} p_{\sigma_i}
  \end{align*}
  Furthermore, bin $b_q$ contains $\sum_{i \in q} a_{2,\sigma_i} p_{\sigma_i} =
  1/(2m) \sum_{i \in q} p_{\sigma_i}$ positive fraction. Using this, we can come up
  with an expression for $\gamma_2$.
  \begin{align*}
    \gamma_2 &= \frac{1}{\rho} \sum_{q \in Q} \p{v_b \cdot \frac{1}{2m} \sum_{i
    \in q} p_{\sigma_i}} \\
    &= \frac{1}{m} \sum_{q \in Q} \frac{1}{|q|} \p{\sum_{i \in q}
    p_{\sigma_i}}^2
  \end{align*}
  Setting this equal to $\gamma$, we have
  \begin{align*}
    \frac{1}{m} \sum_{q \in Q} \frac{1}{|q|} \p{\sum_{i \in q} p_{\sigma_i}}^2
    &= \frac{1}{m} \sum_{i=1}^{2m} p_{\sigma_i}^2 - \frac{1}{m^5} \\
    \sum_{q \in Q} \frac{1}{|q|} \p{\sum_{i \in q} p_{\sigma_i}}^2
    &= \sum_{i=1}^{2m} p_{\sigma_i}^2 - \frac{1}{m^4}
  \end{align*}
  Subtracting both sides from $\sum_{i=1}^{2m} p_{\sigma_i}^2$ and using
  Lemma~\ref{lem:square_diff}, we have
  \begin{equation}
    \sum_{q \in Q} \frac{1}{|q|} \sum_{i < j \in q} (p_{\sigma_i} -
    p_{\sigma_j})^2 = \frac{1}{m^4}
    \label{eq:reduction}
  \end{equation}
  Thus, $Q$ is a fair nontrivial assignment if and only if \eqref{eq:reduction}
  holds.

  Next, we show that there exists $Q$ that satisfies \eqref{eq:reduction} if and
  only if there there exists some $S \subseteq [m]$ such that $\sum_{i \in S}
  \hat w_i = 1/m^4$.

  Assume $Q$ satisfies \eqref{eq:reduction}. Then, we first observe that any $q
  \in Q$ must either contain a single $i$, meaning it does not contribute to the
  left hand side of \eqref{eq:reduction}, or $q = \{2i-1, 2i\}$ for some $i$. To
  show this, observe that the closest two elements of $\{p_{\sigma_1}, \dots,
  p_{\sigma_{2m}}\}$ not of the form $\{p_{\sigma_{2i-1}}, p_{\sigma_{2i}}\}$ must
  be some $\{p_{\sigma_{2i}}, p_{\sigma_{2i+1}}\}$. However, we find that
  \begin{align*}
    (p_{\sigma_{2i+1}} - p_{\sigma_{2i}})^2 &= \p{\frac{i+1}{m+1} -
    \varepsilon_{i+1} - \p{\frac{i}{m+1} + \varepsilon_i}}^2 \\
    &= \p{\frac{1}{m+1} - \varepsilon_{i+1} - \varepsilon_i}^2 \\
    &= \p{\frac{1}{m+1} - \sqrt{\frac{\hat w_{i+1}}{2}} - \sqrt{\frac{\hat
    w_i}{2}}}^2 \\
    &\ge \p{\frac{1}{m+1} - \sqrt{\frac{2}{m^4}} \tag{$\hat w_i \le 1/m^4$}}^2
    \\
    &= \p{\frac{1}{m+1} - \frac{\sqrt{2}}{m^2}}^2 \\
    &\ge \p{\frac{1}{2m} - \frac{\sqrt{2}}{m^2}}^2 \\
    &= \p{\frac{m - 2\sqrt{2}}{2m^2}}^2 \\
    &\ge \p{\frac{m}{4m^2}}^2 \\
    &= \p{\frac{1}{4m}}^2 \\
    &= \frac{1}{16 m^2}
  \end{align*}
  If any $q$ contains any $j,k$ not of the form $2i-1, 2i$, then
  \eqref{eq:reduction} will have a term on the left hand side at least $1/m
  \cdot 1/(16m^2) = 1/(16m^3) > 1/m^4$ for large enough $m$, and since there can
  be no negative terms on the left hand side, this immediately makes it
  impossible for $Q$ to satisfy \eqref{eq:reduction}.

  Consider every $2i-1, 2i \in [2m]$. Let $q_i = \{2i-1, 2i\}$. As shown above,
  either $q_i \in Q$ or $\{2i-1\} \in Q$ and $\{2i\} \in Q$. In the latter case,
  neither $p_{\sigma_{2i-1}}$ nor $p_{\sigma_{2i}}$ contributes to
  \eqref{eq:reduction}. If $q_i \in Q$, then $q_i$ contributes
  $1/2(p_{\sigma_{2i} - p_{\sigma_{2i-1}}})^2 = 1/2 (2\varepsilon_i)^2 = \hat
  w_i$ to the overall sum on the left hand side. Therefore, we can write the
  left hand side of \eqref{eq:reduction} as
  \begin{equation*}
    \sum_{q \in Q} \frac{1}{|q|} \sum_{i < j \in q} (p_{\sigma_i} -
    p_{\sigma_j})^2 = \sum_{q_i \in Q} \frac{1}{2} (p_{\sigma_{2i} -
    p_{\sigma_{2i-1}}})^2 =  \sum_{q_i \in Q} \hat w_i = \frac{1}{m^4}
  \end{equation*}
  Then, we can build a solution to the original subset sum instance as $S = \{i
  : q_i \in Q\}$, giving us $\sum_{i \in S} \hat w_i = \frac{1}{m^4}$.
  Multiplying both sides by $T m^4$, we get $\sum_{i \in S} w_i = T$, meaning
  $S$ is a solution for the subset sum instance.

  To prove the other direction, assume we have a solution $S \subseteq [m]$ such
  that $\sum_{i \in S} w_i = T$. Dividing both sides by $T m^4$, we get
  $\sum_{i \in S} \hat w_i = 1/m^4$. We build a partition $Q$ of $2m$ by
  starting with the empty set and adding $q_i = \{2i-1, 2i\}$ to $Q$ if $i \in
  S$ and $\{2i-1\}$ and $\{2i\}$ to $Q$ otherwise. Clearly, each element of
  $[2m]$ appears in $Q$ at most once, making this a valid partition. Moreover,
  when checking to see if \eqref{eq:reduction} is satisfied (which is true if
  and only if $Q$ is a fair assignment), we can ignore all $q \in Q$ such that
  $|q| = 1$ because they don't contribute to the left hand side. Since, we again
  have
  \begin{equation*}
    \sum_{q \in Q} \frac{1}{|q|} \sum_{i < j \in q} (p_{\sigma_i} -
    p_{\sigma_j})^2 = \sum_{q_i \in Q} \frac{1}{2} (p_{\sigma_{2i} -
    p_{\sigma_{2i-1}}})^2 =  \sum_{q_i \in Q} \hat w_i = \frac{1}{m^4}
  \end{equation*}
  meaning $Q$ is a fair assignment. This completes the reduction.
\end{proof}

We have shown that the integral assignment problem is NP-hard, and it is clearly
in NP because given an integral assignment, we can verify in polynomial time
whether such an assignment satisfies the conditions (A), (B), and (C). Thus, the
integral assignment problem is NP-complete.

\end{document}